\documentclass[twoside]{article}
\usepackage[left=2cm,right=2cm,top=2cm,bottom=2cm]{geometry}
\usepackage{amsmath}
\usepackage{amsfonts}
\usepackage{amssymb}
\usepackage{graphicx}
\usepackage{color}
\usepackage[normalem]{ulem}
\usepackage[utf8]{inputenc}
\usepackage{amsfonts}
\usepackage{color}
\usepackage[normalem]{ulem}
\newenvironment{proof}{\noindent{\sc Proof.}}{\qed}
\newtheorem{theorem}{Theorem}[section]
\newtheorem{lemma}{Lemma}[section]
\newtheorem{cor}{Corollary}[section]
\newtheorem{rem}{Remark}[section]
\newtheorem{definition}{Definition}[section]

\newtheorem{uda}{Example}[section]
\newcommand{\qed}{$\blacksquare$}

\def\bhag#1{\noindent
\setcounter{equation}{0}
\section{#1}
}

\def\RR{{\mathbb R}}

\def\ZZ{{\mathbb Z}}
\def\SS{{\mathbb S}}

\def\x{\mathbf{x}}

\def\y{\mathbf{y}}

\def\w{\mathbf{w}}

\def\O{{\cal O}}

\def\C{{\mathcal C}}

\def\YY{\mathbb{Y}}

\def\be{\begin{equation}}
\def\ee{\end{equation}}
\def\bea{\begin{eqnarray}}
\def\eea{\end{eqnarray}}
\def\eref#1{(\ref{#1})}
\def\disp{\displaystyle}

\def\donchitre#1#2{\vskip 6.5cm\noindent
\parbox[t]{1in}{\special{eps:#1.eps x=6.5cm y=5.5cm}}
\hbox to 7cm{}\parbox[t]{0.0cm}{\special{eps:#2.eps x=6.5cm y=5.5cm}}}

\def\XX{{\mathbb X}}
\def\BB{{\mathbb B}}

\def\supp{\mathsf{supp}}

\title{Dimension independent bounds for general shallow networks}
\author{
 H.~N.~Mhaskar\thanks{
Institute of Mathematical Sciences, Claremont Graduate University, Claremont, CA 91711. The research of this author is supported in part by the Office of the Director of National Intelligence (ODNI), Intelligence Advanced Research Projects Activity (IARPA), via 2018-18032000002.
\textsf{email:} hrushikesh.mhaskar@cgu.edu} 
 }
 \date{}
\begin{document}

\maketitle

\begin{abstract}
This paper proves an abstract theorem  addressing in a unified manner two important problems in function approximation: avoiding curse of dimensionality and estimating the degree of approximation for out-of-sample extension in manifold learning. 
We consider an abstract (shallow) network that
 includes, for example,  neural networks,  radial basis function networks, and kernels  on data defined manifolds used for function approximation in various settings. 
 A deep network is obtained by a composition of the shallow networks according to a directed acyclic graph, representing the architecture of the deep network.

In this paper, 
we prove dimension independent bounds for approximation by shallow networks in the very general setting of what we have called $G$-networks on a compact metric measure space, where the notion of dimension is defined in terms of the cardinality of maximal distinguishable sets, generalizing the notion of dimension of a cube or a manifold.
Our techniques give bounds that improve without saturation with the smoothness of the kernel involved in an integral representation of the target function.
In the context of manifold learning, our bounds provide estimates on the degree of approximation for an out-of-sample extension of the target function to the ambient space.

One consequence of our theorem is that without the requirement of robust parameter selection, deep networks using a non-smooth activation function such as the ReLU, do not provide any significant advantage over shallow networks in terms of the degree of approximation alone.

\end{abstract}

 tractability of integration.

\bhag{Introduction}\label{intsect}

An important problem in machine learning is to approximate a target function $f$ defined on some compact subset of a Euclidean space $\RR^Q$ by a model $P$, e.g., a neural network, radial basis function network, or a kernel based model. 
A central problem in this theory is to estimate the complexity of approximation; i.e., (loosely speaking) to obtain a bound on the number of parameters in $P$ in order to \emph{ensure} that $f$ can be approximated by $P$ within a prescribed accuracy $\epsilon>0$.
Typically, the number of parameters  grows  as a function of $\epsilon^{-1/Q}$; i.e., exponentially with $Q$, a phenomenon known as the curse of dimensionality.

One approach to mitigate the curse of dimensionality is to assume that the data comes from an unknown manifold of a  low dimension $q$ embedded in $\RR^Q$. 
The subject of manifold learning deals with questions of approximation on this manifold, typically based on the eigenfunctions of some differential operator on the manifold or some kernel based methods  (e.g., \cite{rosasco2010learning, rudi2017falkon, sergei_multiple_kernels2017}).
One major problem in this domain of ideas is that the models used for the approximation are based on the manifold alone, which is determined by the existing data. 
Therefore, if a new datum arrives, it might require a change of the manifold, equivalently, starting the computation all over again. 
This is called the problem of out-of-sample extension.
In kernel based methods, it is typically solved using
the so called Nystr\"om extension, but estimating the degree of approximation on the ambient space is an open problem as far as we are aware.

In \cite{poggio2016_cbmm58, dingxuanpap}, we have argued that
deep networks are able to overcome the curse of dimensionality using what we have called the ``blessing of compositionality''.  
We have observed that many functions $f$ of practical interest have a compositional structure.
Although shallow networks cannot take advantage of this fact, deep networks can be built to have the same compositional structure.
For example, we consider a function $F$ of $4$ variables with the structure
$$
F(x_1,x_2,x_3,x_4)=f(f_1(x_1,x_2),f_2(x_3,x_4)).
$$
We then construct shallow networks $P, P_1,P_2$ to approximate the bivariate functions $f, f_1,f_2$ respectively.
Under appropriate assumptions on the smoothness classes of these functions, the number of parameters in the deep network $P(P_1(x_1,x_2),P_2(x_3,x_4))$ required to ensure an accuracy of $\epsilon$ in the approximation of $F$ is $\O(\epsilon^{-r/2})$, where $r$ is a parameter associated with the smoothness of the functions involved. 
In contrast, a shallow network is unable to simulate the compositional structure, and hence, must treat $F$ as a function of $4$ variables. 
The resulting estimate on the number of parameters is then $\O(\epsilon^{-r/4})$.

We note that compositionality is a property of the expression of a function, not an intrinsic property of the function itself.
A simple example in the univariate case is the constant function $f(x)\equiv 2$, $x\in [0,1]$, that can also be expressed as a compositional function
$$f(x)=(x+1)\cosh\left(\log\left(\frac{2+\sqrt{3-2x-x^2}}{x+1}\right)\right), \qquad x\in [0,1].$$
It is therefore natural to ask for which functions shallow networks can already avoid the curse of dimensionality.

The main purpose of this paper is to address the following two problems: (1)  dimension-independent bounds in approximation by shallow networks, (2) approximation bounds for an out-of-sample extension in manifold learning. 
There is a by-product of our results that is of interest in information based complexity.
An important problem in that theory is to obtain bounds on the discretization error for integrals in a high dimensional setting.
Much of the work in this direction (e.g., \cite{dick2010digital}) is focused on integration on a cube (or the whole Euclidean space) with a weight function having a tensor product structure.
Our result proves dimension independent bounds in a very general setting that does not require a tensor product structure, neither in the domain of integration nor for the measure with respect to which the integral is taken.

In Section~\ref{techsect}, we give a more technical introduction, including in precise terms the notion of curse of dimensionality, and a review of some ideas involved in function approximation.
We explain our set-up and discuss the main theorem in Section~\ref{mainsect}. 
The main theorem is illustrated in a number of examples in Section~\ref{appsect} : approximation of functions on the sphere (and hence, the Euclidean space) by networks using ReLU activation functions (Corollary~\ref{relucor}), approximation of functions on the sphere using a class of zonal function networks using a positive definite activation function (Corollary~\ref{zfcor}), approximation of functions on a cube using certain radial basis function networks (Corollary~\ref{cubecor}) and approximation of functions on a manifold and their out-of-sample (Nystr\"om) extensions to the ambient space (Corollary~\ref{manifoldcor}).
The proof of the main theorem is given in Section~\ref{pfsect}.

\bhag{Technical introduction}\label{techsect}
We consider the problem of approximating a function $f$ defined on a compact subset $\XX$ of some Euclidean space $\RR^Q$ by mappings of the form  $x\mapsto\sum_{j=1}^N a_jG(x,y_j)$, where $G :\XX\times\XX\to\RR$ is a kernel (not necessarily symmetric), $x, y_j\in\XX$, and $a_j$'s are real numbers. 
We will refer to such a mapping as a (shallow) $G$-network with $N$ neurons, and denote the class $\{\sum_{j=1}^N a_jG(\circ, y_j) : a_1,\cdots,a_N\in\RR, y_1,\cdots,y_N \in \XX\}$ of all such networks by $\mathsf{span}_N(G)$. 
For example, the action of a neuron, $\sigma(\w\cdot\x'+b)$, can be expressed as $G(\x,\y)=\sigma(\x\cdot\y)$, where $\x=(\x',1)$, and $\y=(\w,b)$, that of a radial basis function by $G(\x,\y)=\Phi(|\x-\y|_2)$, etc.
A deep $G$-network is obtained by composition of such networks according to some directed acyclic graph.

An important problem in this theory is to estimate the  \emph{degree of approximation} to $f$ from $\mathsf{span}_N(G)$, defined by
\be\label{deg_of_app_def}
\mathsf{deg}_N(G;f)=\inf_{a_1,\cdots,a_N, y_1,\cdots,y_N}\left\|f-\sum_{j=1}^N a_jG(\circ,y_j)\right\|_{\mathcal{X}}=\inf_{P\in \mathsf{span}_N(G)}\|f-P\|_{\mathcal{X}},
\ee
where is $\mathcal{X}$ is some Banach space of functions on $\XX$.
In theoretical analysis, one assumes some prior on $f$, encoded by the assumption that $f\in\mathbb{K}$ for some compact subset $\mathbb{K}$ of a Banach space $\mathcal{X}$. The set $\mathbb{K}$ is known in approximation theory parlance as the \emph{smoothness class}. 
A central problem in the theory is then to estimate the \emph{worst case error} 
\be\label{worst_err}
\mathsf{wor}_N(\mathbb{K})=\sup_{f\in \mathbb{K}}\mathsf{deg}_N(G;f).
\ee
From a practical point of view, one seeks a constructive procedure to realize at least sub-optimally the infimum expression in \eref{deg_of_app_def}. 
This is described in abstract terms as follows. Let
the \emph{parameter selector} $\Theta_N : \mathbb{K}\to\RR^N\times \XX^N$  be given by $\Theta_N(f)=(a_{1,N}(f),\cdots,a_{N,N}(f), y_{1,N}(f),\cdots,y_{N,N}(f))$.
If $\Theta_N$ is a continuous map, we say that it is a \emph{robust} parameter selector.
We define the error in  approximation to $f$ using this mapping by
\be\label{degdef}
\mathsf{err}_N(f,\Theta_N)=\left\|f-\sum_{j=1}^N a_{j,N}(f)G(\circ,y_{j,N}(f))\right\|_{\mathcal{X}}.
\ee
 Instead of $\mathsf{wor}_N(\mathbb{K})$ one seeks to estimate
\be\label{opt_recovery}
\mathsf{opt}_N(\mathbb{K})=\inf_{\Theta_N}\sup_{f\in \mathbb{K}}\mathsf{err}_N(f,\Theta_N(f)),
\ee
where the infimum is taken over all robust parameter selectors $\Theta_N$.

We note that in the expression for $\mathsf{deg}_N(G;f)$, the  parameters $a_j$ and $y_j$ are allowed to be selected adaptively on the target function $f$ involved. 
In contrast, the definition of $\mathsf{err}_N$ involves a fixed parameter selection procedure all $f\in \mathbb{K}$.
Therefore,  $\mathsf{deg}_N(G;f)\le \mathsf{err}_N(f,\Theta_N)$ for all parameter selectors $\Theta_N$.  
It is not clear whether for every $f\in \mathbb{K}$, there exists a unique best approximation from the space $\mathsf{span}_N(G)$. 
If this is the case, let the unique best approximation to $f$ be $\sum_{j=1}^N a_{j,N}^*(f)G(\circ,y_{j,N}^*(f))$, and $\Theta_N^*(f)=(a_{1,N}^*(f),\cdots,a_{1,N}^*(f), y_{1,N}^*(f), \cdots, y_{N,N}^*(f))$. 
Then by definition, $\mathsf{deg}_N(G;f)=\mathsf{err}_N(f,\Theta_N^*)$ for every $f\in \mathbb{K}$.
If it can be proved that $\Theta_N^*$ is also a continuous mapping on $\mathbb{K}$, then 
$\mathsf{wor}_N(\mathbb{K})= \mathsf{opt}_N(\mathbb{K})$.
We note, however, that the issue of existence of best approximation, its uniqueness, and the continuity of the parameters involved are not immediately obvious even in the most classical case of polynomial approximation on an interval.

We digress to make a note on terminology. 
The term degree of approximation of $f$ from $\mathsf{span}_N(G)$ is defined by \eref{deg_of_app_def}.
However, the quantity $\mathsf{err}_N(f,\Theta_N)$ is also referred to as the degree of approximation to $f$ by networks prescribed by the summation expression in \eref{degdef}. 
The terms error in approximation (or approximation error) are also used to indicate degree of approximation.
The terms rate (or accuracy) of approximation refers to an upper estimate on the degree of approximation.

Many classes $\mathbb{K}$ used in this theory suffer from the so-called curse of dimensionality (cf. \cite{devore1989optimal}) :
\be\label{curse}
\mathsf{opt}_N(\mathbb{K}) \ge cN^{-r/Q},
\ee
where $r$ is a ``smoothness parameter'' associated with $\mathbb{K}$.
The curse of dimensionality is avoided either by assuming a different prior on the target function or by dropping the requirement that the parameter selector be robust.
The purpose of this paper is to explore the second option.
We will show that even  the smoothness classes typically studied in the literature that give rise to the curse of dimensionality do not suffer from the same if we do not require the parameter selector to be robust.
This is observed in \cite{devore1989optimal}, where there was no restriction on the parameter selector and the recovery algorithm, so that a space-filling curve could be used in theory.
In our setting, the parameter selector has a specific meaning and the recovery algorithm consists of constructing a $G$-network using these parameters. 

In order to motivate our work, we first review some of the ideas in the existing work on the estimation of degree of approximation by shallow networks.

First, it is clear that if the parameter selector $\Theta_N$ is robust, then $\sum_{k=1}^N |a_{j,N}(f)|\le c_N$, where $c_N>0$ is a constant independent of $f\in\mathbb{K}$. 
It is sometimes assumed (or even proved under suitable conditions) in the literature that $c_N$ can be chosen independent of $N$ as well (e.g., \cite{schmidt2017nonparametric, schmidt2019deep}).
Then it is easy to see that in order for the sequence of networks to converge to $f$, it is necessary that $f$ must admit a representation of the form
\be\label{basic_representation}
f(x)=\int_\XX G(x,y)d\tau(y)
\ee
for some (signed) measure $\tau$ having a bounded total variation on $\XX$.
This total variation has been referred to as the $G$-variation  of $f$ \cite{kurkova1, kurkova2}.
Using probabilistic estimates, it is then possible to obtain  the bound $\mathsf{deg}_N(G;f)= \O(\sqrt{\log N/N})$.
Many other results of this type have been obtained in the literature (e.g., \cite{barron1992neural, barron1993, klusowski2016uniform, tractable, kuurkova2018constructive}). 
All of these either assume explicitly or deduce from the assumptions in these papers that a representation of the form \eref{basic_representation} holds. 
Also, the error bounds neither require nor  depend upon the smoothness of $G$.

A representation of the form \eref{basic_representation} holds also for many classes for which the curse of dimensionality applies.
For example, let $\XX$ be the unit sphere $\SS^Q$ of the Euclidean space $\RR^Q$, $\Delta$ be the (negative) Laplace-Beltrami operator on $\SS^Q$, $r\ge 1$ be an integer, and we consider $\mathbb{K}$ to be the class of all continuous functions $f$ on $\SS^Q$ for which $(I+\Delta)^r f$ is continuous.
The so-called non-linear $N$ width for this class is $\sim N^{-r/Q}$ (cf. \cite{lizorkin1994nikol}). 
However, if $G$ is the Green function for the operator $(I+\Delta)^r$, then every $f\in \mathbb{K}$ has a representation of the form 
$$
f(\x)=\int_{\SS^Q}G(\x,\y)((I+\Delta)^r f)(\y)d\mu^*(\y),
$$
where $\mu^*$ is the volume element of $\SS^Q$. 
Indeed, this fact is used critically in our work \cite{zfquadpap} on approximation by zonal function networks using $G$ as the activation function, where we gave explicit constructions based entirely on the data $\{(\x_j, f(\x_j))\}$ with no stipulations on the locations where the sampling nodes $\x_j$ are.
Similar representations play a critical role in similar estimates in approximation theory (e.g., \cite[Chapter~7, Section~4]{devlorbk}), including many papers of ours, e.g., \cite{eignet, sphrelu}.
In some sense, this is the other extreme of the kind of results on the degree of approximation, where the smoothness of $G$ is the only determining factor.
Clearly, probabilistic ideas can be used to obtain dimension independent bounds instead, if only we give up the requirement of a robust parameter selection.
However, the challenge here is not to loose the advantage offered by the smoothness of $G$.

To summarize, in both of these approaches, one has an integral representation of the form \eref{basic_representation}, but  get different bounds depending upon the norm and the method used.

In this paper, we will consider a very general set-up where $\XX$ can be an arbitrary compact metric measure space, and consider functions that admit a representation of the form \eref{basic_representation}. 
Giving up the requirement of robust parameter selector, we will use an idea in the paper \cite{bourgain1988distribution} of Bourgain and Lindenstrauss to obtain dimension independent bounds (cf. Definition~\ref{dimensiondef}) in the uniform norm provided some very mild conditions hold.
This technique involves aspects from both the approaches mentioned above.
Thus, we will use concentration inequalities as in the first approach.
Our conditions on $G$ will be in terms of approximation of $G$ using a fixed basis as in the second approach.

We will elaborate more about the highlights below in the paper at appropriate places, but they can be summarized as follows.
\begin{itemize}
\item Our bounds are in the uniform norm. 
We have argued in \cite{mhaskar2018analysis} that the usual measurement of generalization error using the expected value of the least square loss is not applicable for approximation theory for deep networks; one has to use the uniform approximation to take full advantage of the compositional structure. 
Moreover, results about shallow networks can then be ``lifted'' to those about deep networks using a property called good propagation of errors.
\item Our results combine the advantages of the probabilistic approach to obtain dimension independent bounds and the classical approximation theory approach where  the higher the smoothness of the activation function, the better the bounds on the degree of approximation.
\item We allow the measure $\tau$ to be, for example, supported on a sub-manifold $\YY$ of a manifold  $\XX$. 
Under certain conditions, the constants involved in the bounds on the degree of approximation depend upon the  sub-manifold $\YY$ alone.
\item At the same time, taking $G$ to be a kernel well defined on the ambient space $\XX$, our bounds provide estimates on the degree of approximation for the out-of-sample extension of the target function to the entire space $\XX$. 
The asymptotic behavior of these bounds is  also  independent of the dimension of the ambient space, but the constants may depend upon the dimension of the ambient space. 
\end{itemize}

\bhag{The set-up and main theorem}\label{mainsect}
In Section~\ref{basicsect}, we formulate our general setting. 
The notion of dimension is developed in Section~\ref{dimensionsect}.
The notion of local smoothness of a function on a metric space, and the class of kernels that enter into \eref{basic_representation} is described in Section~\ref{kernelsect}. 
In Section~\ref{measuresect}, we introduce the measure theoretic concepts concerning the class of measures we wish to use in \eref{basic_representation}.
With this preparation, the main theorem is stated and discussed in Section~\ref{theoremsect}.

\subsection{Basic set-up}\label{basicsect}
Let $\XX$ be a compact metric space, $\rho$ be the metric on $\XX$, and $\mu^*$ be a probability measure on $\XX$.
If $x\in\XX$, $\delta>0$, the ball of radius $\delta$ centered at $x$ is denoted by $\mathbb{B}(x,\delta)$; i.e.,
\be\label{balldef}
\mathbb{B}(x,\delta)=\{y\in\XX : \rho(x,y)\le \delta\}.
\ee
If $A\subseteq \XX$, it is convenient to denote $\mathbb{B}(A,\delta)=\bigcup_{x\in A}\mathbb{B}(x,\delta)$.
We will denote the closure of $\XX\setminus \mathbb{B}(A,\delta)$ by $\Delta(A,\delta)$.

In the sequel, we assume that there exist $Q,\kappa_1,\kappa_2>0$ such that
\be\label{ballmeasure}
\kappa_1\delta^Q\le \mu^*(\BB(x,\delta))=\mu^*\left(\{y\in \XX: \rho(x,y)<\delta\}\right)\le \kappa_2\delta^Q, \qquad x\in\XX,\ 0< \delta\le 1.
\ee
As the examples below show, $Q$ serves as a dimension parameter for the ambient space $\XX$.
In Definition~\ref{dimensiondef}, we will define the notion of dimension more formally, without requiring a measure. 

If $A\subseteq \XX$, the symbol $C(A)$ denotes the class of bounded, real valued, uniformly continuous functions on $A$, equipped with the supremum norm $\|\cdot\|_A$. We will omit the subscript $A$ if $A=\XX$, and write $\|\cdot\|=\|\cdot\|_\XX$.
Let $\{\Pi_k\}$ be a nested sequence of finite dimensional subspaces of $C(\XX)$: $\Pi_0\subset \Pi_1\subset \Pi_2\subset\cdots$, with the dimension of $\Pi_k$ being $D_k$, $k\in\ZZ_+=\{0,1,\cdots\}$.\\

\noindent\textbf{Constant convention:}\\
\textit{
In the sequel, the symbols $c, c_1,\cdots$ will denote generic positive constants depending only on the  fixed quantities under discussion, such as the smoothness parameters, $\kappa_1$, $\kappa_2$, $Q$, the dimensions, etc. 
Their values may be different at different occurrences, even within a single formula. The notation $A\sim B$ means $c_1A\le B\le c_2A$. \qed}\\

\begin{uda}\label{sphere_set_up}
{\rm
Let 
$$
\XX=\SS^Q=\{(x_1,\cdots,x_{Q+1})=\x\in\RR^{Q+1} : |\x|_2^2=\sum_{k=1}^{Q+1} x_k^2=1\}.
$$
We let $\rho$ be the geodesic distance on $\XX$, $\mu^*$ be the volume measure on $\XX$, normalized to be a probability measure. 
The space $\Pi_n=\Pi_n^Q$ is the space of all spherical polynomials of degree $<n$; i.e., the restriction to $\SS^Q$ of algebraic polynomials in $Q+1$ variables of total degree $<n$. 
The dimension of $\Pi_n$ is  $\O(n^Q)$.
\qed}
\end{uda}

\begin{uda}\label{cube_set_up}
{\rm
Let $\XX= [-1,1]^Q$, $\mu^*$ being the Lebesgue measure, normalized to be a probability measure, $\rho(\x,\y)=|\x-\y|_2$.
Let $\Pi_n$ be the class of all polynomials of total degree $<n$. The dimension of $\Pi_n$ is $\O(n^Q)$.
\qed}
\end{uda}

\begin{uda}\label{gen_measure_space}
{\rm
It is possible to convert any measure space that admits a non-atomic measure into a compact metric measure space with the properties as described above. 
Let $\XX$ be any measure space with a non-atomic probability measure $\mu^*$; i.e., for any measurable $A\subset \XX$ with $\mu^*(A)>0$, there exists a measurable subset $B\subset A$ with $0<\mu^*(B)<\mu^*(A)$. 
Then using ideas described in \cite[Chapter~VIII, Section~40]{halmos2013measure}, we obtain a nested sequence $\{A_{k,n}\}_{k=0}^{2^{n-1}}$ of partitions  of $\XX$ such that $\mu^*(A_{k,n})=2^{-n}$, $A_{0,0}=\XX$, and for $n\ge 1$, each $A_{k,n}\subset A_{\lfloor k/2\rfloor,n-1}$, where $\lfloor k/2\rfloor$ is the integer part of $k/2$. 
For each $n\ge 0$, we can then associate $A_{k,n}$ with an interval of the form $I_{k,n}=[k/2^n, (k+1)/2^n)$. 
Thus,  every point in $\XX$ corresponds to a unique number $\varepsilon(x)$ of the form $\sum_{j=0}^\infty a_j(x)2^{-j-1}$, where each $a_j(x)\in \{0,1\}$, and an infinite tail of $1$'s is prohibited in the sequence $\{a_j(x)\}$. 
We define an equivalence class on $\XX$ by writing $x\sim y$ if $\varepsilon(x)=\varepsilon(y)$, and replace $\XX$ by its corresponding quotient space, so that the correspondence $\varepsilon$ is one-to-one.
We note that for any measurable subset $A\subseteq \XX$, $\mu^*(A)$ is preserved under this operation.
We define a metric on $\XX$ by
\be\label{metricdef}
\rho(x,y)=\sum_{j=0}^\infty (a_j(x)\oplus a_j(y))2^{-j-1},
\ee
where $\oplus$ denotes the exclusive or between the digits.
Then $\XX$ is a compact metric space with this metric.
It is not difficult to verify that if $x\in A_{k,n}$ (equivalently, $\varepsilon(x)\in I_{k,n}$),  $y\in\XX$, and $\rho(x,y)<1/2^n$ then $y\in A_{k,n}$ as well. Conversely, if $x,y\in A_{k,n}$ then $\rho(x,y)\le 1/2^n$. So, by the construction of the partition $\{A_{k,n}\}$, \eref{ballmeasure} is satisfied with $Q=1$.
A nested sequence of subspaces of  $C(\XX)$ can then be constructed using the ideas in \cite{treepap}, where the question of degree of approximation is also considered in detail.
However, since the ``dimension parameter'' $Q=1$ in this case, we will not pursue this example further in this paper.
\qed}
\end{uda}

\subsection{Dimension of a family of sets}\label{dimensionsect}
Next, we define some abstract ideas, starting with the notion of a dimension for a subset of $\XX$. 
Specific examples will be given in detail in Section~\ref{appsect}.

For a finite subset $\C\subseteq \XX$ with $|\C|\ge 2$, and a compact subset $K\subseteq \XX$, we write
\be\label{meshnormdef}
\delta(\C;K)=\sup_{x\in K}\min_{y\in\C}\rho(x,y), \qquad \eta(\C)=\min_{y,z\in\C, y\not=z}\rho(y,z).
\ee
We will omit the mention of $K$ if $K=\XX$. 
Let $\epsilon>0$. 
A finite subset $\C\subseteq K$ is called  $\epsilon$-distinguishable if $\eta(\C)\ge \epsilon$. 
It is easy to check that if $\C$ is a maximal $\epsilon$-distinguishable subset of $K$ then
\be\label{netcover}
K \subseteq \bigcup_{y\in\C}\BB(y,\epsilon), \qquad \BB(y,\epsilon/3)\cap \BB(z,\epsilon/3)=\emptyset, \mbox{ if } y,z\in\C, \ y\not=z.
\ee
In particular, $\eta(\C)=\delta(\C;K)=\epsilon$. 
Moreover,   using a volume argument and \eref{ballmeasure}, we see that if $K=\XX$, then $\kappa_2^{-1}\epsilon^{-Q}\le |\C|\le 3^Q\kappa_1^{-1} \epsilon^{-Q}$.

If $A\subseteq\XX$, $\epsilon>0$,  we denote by $H_\epsilon(A)$ the number of points in a maximal $\epsilon$-distinguishable set for the closure of $A$.

\begin{definition}\label{dimensiondef}
Let $d\ge 0$. A family  $\mathfrak{F}$  of subsets of $\XX$ is called (at most) $d$-dimensional if there exists a constant $c(\mathfrak{F})<\infty$ such that
$$
\sup_{A\in\mathfrak{F}}H_{\epsilon}(A)\le c(\mathfrak{F})\epsilon^{-d}, \qquad 0<\epsilon<1.
$$
A subset $A\subset \XX$ is $d$-dimensional if $\{A\}$ is $d$-dimensional.
\end{definition}

\subsection{Local smoothness and kernels}\label{kernelsect}

We need to define a local smoothness for the kernel $G$ that we wish to use in \eref{basic_representation}.
In classical wavelet analysis and theory of partial differential equations, it is customary to define the local smoothness of a function at a point $x$ in terms of the degree of approximation of the function  by polynomials of  a fixed degree over the neighborhoods of $x$, measured in terms of the diameters of these neighborhoods.
In our analysis, we will use the spaces $\Pi_n$ for this purpose.
However, our definition is a bit more complicated in the absence of any structure on $\XX$ and detailed spline-like approximation theory for the spaces $\Pi_n$.

If $A\subset \XX$, $f\in C(A)$, let
\be\label{degapprox}
E_n(A;f)=\inf_{P\in\Pi_n}\|f-P\|_A.
\ee

\begin{definition}\label{loc_smooth_def}
Let $r>0$, $A\subseteq \XX$.  A function $f\in C$ is called \textbf{$r$-smooth  on $A$} if there exists $d(A)>0$ such that,
\be\label{loc_lip_norm}
\|f\|_{A;r}=\sup_{x\in A}\sup_{0<s\le r}\sup_{\stackrel{y\in \BB(x,d(A))}{ 0<\delta\le d(A)} }\frac{E_s(\mathbb{B}(y,\delta);f)}{\delta^s} <\infty.
\ee
\end{definition}

In formulating the conditions on our kernel $G$, we are motivated  primarily by two examples.  
\begin{uda}\label{relumotivationexamp}
{\rm
We consider the case when $\XX=\SS^Q$, and $G(\x,\y)=|\x\cdot\y|^{2\gamma+1}$, $\x,\y\in \SS^Q$ for some $\gamma\ge 0$ such that $2\gamma+1$ is not an even integer. 
 For each $\x\in\SS^Q$, $G(\x,\circ)$ is $2\gamma+1$ smooth on $\SS^Q$. 
If  $\gamma$ is an integer ($\gamma=0$ corresponds to the ReLU function), then outside $\mathcal{E}_\x=\{\y : \x\cdot\y=0\}$, a set of dimension $Q-1$, $G(\x,\cdot)$ is a spherical polynomial of degree $2\gamma+1$.
 Therefore, for any set $A\subset \SS^Q\setminus \mathcal{E}_x$, $\|G(\x,\circ)\|_{A,R}=0$ for every $R\ge 2\gamma+1$.
If $2\gamma+1$ is not an integer, then for any such set $A$ and $R\ge 2\gamma+1$, $G(\x,\circ)$ is $R$ times differentiable, but $\|G(\x,\circ)\|_{A,R}\le c\mathsf{dist}(A,\mathcal{E}_\x)^{2\gamma+1-R}$.\qed}
\end{uda}

\begin{uda}\label{lapacemotivationexamp}
{\rm
We consider the case $\XX=[-1,1]^Q$, $G(\x,\y)=\exp(-|\x-\y|_2)$. 
It is clear that for each $\x\in [-1,1]^Q$, $G(\x,\circ)$ is $1$-smooth on $[-1,1]^Q$. 
Except for $\y=\x$; i.e., except on a set of dimension $0$, it is also infinitely differentiable.
However, for any $R>1$, and $A\subset [-1,1]^Q\setminus\{\x\}$, $\|G(\x,\circ)\|_{A,R}\le c\mathsf{dist}(A,\{\x\})^{-R}$.\qed}
\end{uda}

\begin{definition}\label{kernelcond}
Let $R\ge r>0$, $\alpha>0$, $F :(0,1]\to [0,\infty)$ be a non-increasing function. A function $G\in C(\XX\times\XX)$ will be called a \textbf{kernel of class $\mathcal{G}(\alpha,r,R, F)$} if each of the following conditions is satisfied.
\begin{enumerate}
\item (\textbf{H\"older continuity})
\be\label{lipschitzcond}
\|G(x,\circ)-G(x',\circ)\| \le c(G)\rho(x,x')^\alpha, \qquad x,x'\in \XX,
\ee
\item (\textbf{Global smoothness}) $G(x,\circ)$ is $r$-smooth on $\XX$, with 
\be\label{globalsmooth}
\sup_{x\in\XX}\|G(x,\circ)\|_{\XX, r} <\infty.\ee
\item (\textbf{Smoothness in the large}) For every $x\in \XX$, there exists a compact set $\mathcal{E}_x=\mathcal{E}_x(G)\subseteq \XX$ with the following property. For every $\delta >0$, $G(x,\circ)$ is $R$-smooth on $\Delta(\mathcal{E}_x,\delta)$ with
\be\label{sobolevnormbd}
\sup_{x\in\XX}\|G(x,\circ)\|_{\Delta(\mathcal{E}_x,\delta),R}\le F(\delta)<\infty.
\ee
\end{enumerate}
\end{definition}
\begin{uda}\label{gexamp}
{\rm
In Example~\ref{relumotivationexamp}, $\mathcal{E}_\x=\{\y\in \SS^Q : \x\cdot\y=0\}$. 
If $\gamma$ is an integer, then we may choose $F(\delta)\equiv c$ for some constant independent of $R$.
Otherwise, $F(\delta)=c\delta^{2\gamma+1-R}$ with $c$ depending upon $Q, \gamma, R$.
In Example~\ref{lapacemotivationexamp}, $\mathcal{E}_\x=\{\x\}$, and we may choose $F(\delta)=c\delta^{-R}$, with $c$ depending on $Q$ and $R$.
\qed}
\end{uda}

\begin{rem}\label{relaxrem}
{\rm 
Since we do not assume $G$ to be symmetric, it is possible to define it on $\XX\times\YY$ instead of $\XX\times\XX$, where $\YY$ is another compact metric measure space with a measure satisfying a condition analogous to \eref{ballmeasure}. 
The conditions on $G(x,\circ)$ can be formulated for $\YY$ instead of $\XX$.
This will only complicate the presentation of the paper without adding any new insights.
Therefore, we will use the above definition, but in fact, will be applying it with the restriction of $G$ to $\XX\times\YY$, where $\YY$ is the support of a measure on $\XX$.
\qed}
\end{rem}

If $H :(0,1]\to (0,\infty)$ is a non-increasing function, we define 
$$
H^{-1}(t)=\inf\{u : H(u)\le t\}.
$$
\subsection{Measures}\label{measuresect}
We introduce next the conditions on the measures that we wish to use in \eref{basic_representation}.
Before stating our measure theoretic notions, we recall some preliminaries. The term measure will mean a signed or positive Borel measure on $\XX$. The total variation measure $|\tau|$ of a signed measure $\tau$ on $\XX$ is defined by 
$$
|\tau|(A)=\sup\sum_{j}|\tau(U_j)|,
$$
where the sum is over all countable partitions of $A$ into Borel measurable sets $U_j$.
We will denote $|\tau|(\XX)=\|\tau\|_{TV}$.
The support $\supp(\tau)$ is the set of all $x\in\XX$ for which $|\tau|(\BB(x,\delta))>0$ for every $\delta>0$. It is easy to see that $\supp(\tau)$ is a compact subset of $\XX$.

\begin{definition}\label{measuredef}
Let $q>0$. A measure $\tau$ on $\XX$ will be called \textbf{$q$-admissible} if $\tau$ has a bounded total variation $\|\tau\|_{TV}<\infty$, $\supp(\tau)$ is $q$-dimensional subset of $\XX$, and 
\be\label{taucond}
|\tau|(\BB(x,\delta)) \le c\delta^q\|\tau\|_{TV}, \qquad x\in\XX, \ 0<\delta\le 1.
\ee
\end{definition}

\subsection{Main theorem}\label{theoremsect}
Our main theorem can now be stated as follows.
\begin{theorem}\label{maintheo}
Let $q>0$, and $\tau$ be a $q$-admissible measure on $\XX$.
Let  $0\le s\le q$, $R\ge r$, $\alpha>0$, $F: (0,1]\to [0,\infty)$ be non-increasing, $G\in \mathcal{G}(\alpha,r,R, F)$, and for each $x\in\XX$, $\{\supp(\tau)\cap \mathcal{E}_x(G)\}$ be an $s$-dimensional family of subsets of $\XX$. With 
$\tilde{F}(t)=F(t)/t^{(q-s)/2}$, we assume that $\tilde{F}(t)\to\infty$ as $t\downarrow 0$, and define for $n\ge 1$,
\be\label{epsilondef}
\epsilon_n^* = \max (1/n, (\tilde{F})^{-1}(n^{R-r})).
\ee
Let $f: \XX\to\RR$ be defined by
\be\label{functiondef}
f(x)=\int_\XX G(x,y)d\tau(y), \qquad x\in\XX.
\ee
Then for $n\ge c$, there exists an integer $N\sim n^q$,  points $\{y_1,\cdots,y_N\}$ with $\delta(\{y_1,\cdots,y_N\};\supp(\tau))\le 1/n$, and numbers $a_1,\cdots, a_N$ with $|a_k|\le (c/N)\|\tau\|_{TV}$, $k=1,\cdots, N$, such that
\be\label{approxbd}
\left\|f-\sum_{k=1}^N a_kG(\circ,y_k)\right\|_\XX \le c\left(\frac{\log n -\log\epsilon_n^* }{n^{q+2r}(\epsilon_n^*)^{(s-q)}}\right)^{1/2}\|\tau\|_{TV} \le c\left(\frac{\log N-\log \epsilon_{cN^{1/q}}^* }{N^{1+2r/q}(\epsilon_{cN^{1/q}}^*)^{(s-q)}}\right)^{1/2}\|\tau\|_{TV}.
\ee
Here, all the constants involved depend upon  $s, r, q, \alpha, G$, in addition to the other fixed parameters in the definition of $\XX$, in particular, on $Q$.
\end{theorem}

\begin{rem}\label{special_case_rem}
{\rm
In rightmost expression in \eref{approxbd} takes an illustrative form in the case when $F(t)=ct^{\Gamma-R}$ for some $\Gamma\in [0,R]$. 
In this case, $\tilde{F}(t)=ct^{(2\Gamma-2R-(q-s))/2}$, and
$$
\epsilon_{n}^*=\max(1/n,c_1n^{-(2(R-r)/(q-s+2R-2\Gamma)})).
$$
Simplifying, and writing
$$
T=\frac{q-s}{2}\min\left(1,\frac{R-r}{R-\Gamma+(q-s)/2}\right),
$$
we obtain the estimate
\be\label{expressive_est}
\left\|f-\sum_{k=1}^N a_kG(\circ,y_k)\right\|_\XX \le  c\left(\frac{\log N}{N}\right)^{1/2}N^{-r/q}N^{-T/q}\|\tau\|_{TV}.
\ee
The first term in the estimate is the familiar dimension independent bound, the second is the familiar bound for approximation of smooth functions depending upon the global smoothness of $G$, and the last is a correction term to allow for our general set-up.
}
\end{rem}
\begin{rem}\label{very_good_rem}
{\rm
Since $G(x,\circ)$ is $r$-smooth on $\XX$, we may choose $\mathcal{E}_x=\XX$, $F\equiv c$, $s=q$, $R=r=\Gamma$ (as in Remark~\ref{special_case_rem}), and obtain the upper bound in \eref{approxbd} to be $\O((\sqrt{\log N})N^{-1/2-r/q})$. 
In particular, if the function $G(x,\circ)$ is $r$-smooth on $\XX$ for every $r$, then there is no saturation in the degree of approximation.
For any $S>0$, we may take  $\mathcal{E}_x=\XX$,   $s=q$, $R=r=(S+1/2)q$, and obtain the bound $\O((\sqrt{\log N})N^{-S})$.
\qed}
\end{rem}

\begin{rem}\label{tractability_rem}
{\rm (\textbf{Tractability of integration})
Theorem~\ref{maintheo} has another interesting consequence, perhaps, not relevant directly to the theme of the present paper. 
In information based complexity, one is interested in approximating integrals of the form $\int_\XX g(y)d\tau(y)$ for high dimensional spaces $\XX$ so as to obtain the error in the approximation independent of the dimension, except possibly for the constant factors involved.
A great deal of research is devoted to this subject, e.g. \cite{dick2010digital}, where $\XX$ is considered to be a cube and $\tau$ is the Lebesgue measure.
A typical assumption on the class of functions for which these results are applicable also involve a representation of the form
\be\label{gdef}
g(y)=\int_\XX G(x,y)d\nu(x), \qquad y\in\XX,
\ee
for some  measure $\nu$ supported on $\XX$, and having a bounded total variation.
A simple application of Fubini's theorem  leads to 
$$
\int_\XX g(y)d\tau(y)=\int_\XX \left(\int_\XX G(x,y)d\tau(y)\right)d\nu(x).
$$
We may approximate the function of $x$ defined in the parenthesis above using Theorem~\ref{maintheo}, obtaining an approximation of the form $\sum_{k=1}^N a_k\int_\XX G(x,y_k)d\nu(x)=\sum_{k=1}^N a_k g(y_k)$, where $a_k$'s and $y_k$'s depend only on $\tau$ and not on $\nu$; i.e., are independent of $g$. We formulate this observation in the following corollary.
The examples in Section~\ref{appsect} will clarify the choice of $\epsilon_n$ for different kernels.

\begin{cor}\label{tractablecor}
We assume the set up  as in Theorem~\ref{maintheo}, $\mathfrak{G}$ be the set of functions $g$ satisfying \eref{gdef} for some measure $\nu$ with $\|\nu\|_{TV}\le 1$. 
Then for every $N\ge c$, there exist $\{y_1,\cdots,y_N\}$ with $\delta(\{y_1,\cdots,y_N\};\supp(\tau))\le 1/n$, and numbers $a_1,\cdots, a_N$ such that
\be\label{tractableest}
\sup_{g\in \mathfrak{G}}\left|\int_\XX g(y)d\tau(y) - \sum_{k=1}^N a_k g(y_k)\right| \le c\left(\frac{\log N-\log \epsilon_{cN^{1/q}}^* }{N^{1+2r/q}(\epsilon_{cN^{1/q}}^*)^{(s-q)}}\right)^{1/2}\|\tau\|_{TV}.
\ee
\end{cor}
In contrast to much of the literature on this subject, we note that there is no tensor product structure required here. 
Moreover, as explained above, the estimates improve without saturation as the smoothness of $G$ increases.
\qed}
\end{rem}

\bhag{Examples and applications}\label{appsect}
In this section, we illustrate the implications of Theorem~\ref{maintheo} using a number of examples.

\subsection{Approximation by ReLU networks}\label{reluexample}
We assume the set-up described in Example~\ref{sphere_set_up}.  It has been observed in \cite{bach2014, dingxuanpap} that approximation on a compact subset of a Euclidean space $\RR^Q$ by ReLU networks is equivalent to the approximation of an even function on $\SS^Q$ by networks of the form $\x\mapsto\sum_k a_k|\x\cdot\y_k|$. 
(We note that for all $t\in\RR$, $|t|=t_++(-t)_+$, $t_+=(1/2)(|t|+t)$.)
An estimate on the degree of approximation in this context is obtained in \cite{sphrelu} for functions that admit an integral representation as required in Theorem~\ref{maintheo}. Our methods are constructive using a robust parameter selector, and yield a bound of the form $\O(N^{-2/Q})$. (cf. \cite{yarotsky2018optimal} for the optimality of this bound.)
We have considered in \cite{sphrelu} a slightly more general class of activation functions $G(\x,\y)=|\x\cdot\y|^{2\gamma+1}$, $2\gamma+1$ not an even integer, so that the case $\gamma=0$ corresponds to the approximation using ReLU networks.
On the $(Q-1)$-dimensional family of sets $\mathcal{E}_\x=\{\y\in\SS^Q : \x\cdot\y=0\}$, $G$ is $(2\gamma+1)$-smooth, and is infinitely differentiable on $\SS^Q\setminus \mathcal{E}_\x$. 
Clearly, if $\tau$ is any measure and $\supp(\tau)$ is a $q$-dimensional set, the family $\{\mathcal{E}_\x\cap \supp(\tau)\}$ is either $(q-1)$-dimensional or $q$-dimensional. 
In Definition~\ref{kernelcond}, we may take $F(t)\equiv c$ if $\gamma$ is an integer and $F(t)= ct^{2\gamma+1-R}$ otherwise.
Therefore, if $\gamma$ is an integer, we may choose $\epsilon_n^*=1/n$. Otherwise, for any $\beta\in (0,1)$, we may choose $R>2\gamma+1+\beta(q-s)/(2-2\beta)$ with $s=q-1$ or $s=q$ as applicable, and set $\epsilon_n^*= cn^{-\beta}$.
Therefore, Theorem~\ref{maintheo} yields the following corollary.
\begin{cor}\label{relucor}
Let $\gamma\ge 0$, $2\gamma+1$ not an even integer, $0<\beta<1$. We use Theorem~\ref{maintheo} with $G(\x,\y)=|\x\cdot\y|^{2\gamma+1}$. 
Let $\{\mathcal{E}_\x\cap \supp(\tau)\}$ be $s$-dimensional, where $s=q-1$ or $s=q$. (If $q=Q$ then $s=Q-1$).\\
{\rm (a)} If $\gamma$ is an integer, then \eref{approxbd} takes the form
\be\label{reluest}
\left\|f-\sum_{k=1}^N a_kG(\circ,y_k)\right\|_{\SS^Q} \le c\|\tau\|_{TV}\times
\begin{cases}
\disp\frac{\sqrt{\log N}}{N^{1/2+(4\gamma+3)/(2q)}} & \mbox{ if $s=q-1$},\\[2ex]
\disp\frac{\sqrt{\log N}}{N^{1/2+(2\gamma+1)/q}}, &\mbox { if $s=q$}.
\end{cases}
\ee
{\rm (b)} If $\gamma$ is not an integer, then \eref{approxbd} takes the form
\be\label{reluest1}
\left\|f-\sum_{k=1}^N a_kG(\circ,y_k)\right\|_{\SS^Q} \le c\|\tau\|_{TV}\times
\begin{cases}
\disp\frac{\sqrt{\log N}}{N^{1/2+(4\gamma+2+\beta)/(2q)}} & \mbox{ if $s=q-1$},\\[2ex]
\disp\frac{\sqrt{\log N}}{N^{1/2+(2\gamma+1)/q}}, &\mbox { if $s=q$}.
\end{cases}
\ee
\end{cor}
\begin{rem}\label{relurmk}
{\rm
We note that for the ReLU network, $\gamma=0$. If $q=Q$, we may apply the first estimate in \eref{reluest} to obtain the degree of approximation $\O(N^{-(Q+3)/(2Q)})$.
\qed}
\end{rem}

\subsection{Approximation by certain zonal function networks}\label{zfexample}
We assume the set-up described in Example~\ref{sphere_set_up}, and let $G(\x,\y)=(1-\x\cdot\y)^\gamma$, $\gamma$ not an integer. For the  class of functions satisfying \eref{functiondef}, with $q=Q$, the error in approximation  with a robust parameter selector and completely constructive procedure given in \cite{zfquadpap} is $\O(N^{-2\gamma/Q})$. 
In order to apply Theorem~\ref{maintheo}, we note that $\mathcal{E}_x=\{x\}$, $s=0$. 
For any $R>2\gamma$, $F(t)=ct^{2\gamma-R}$. 
Therefore, for any $\beta\in (0,1)$, we may choose $R> 2\gamma+\beta q/(2-2\beta)$, and $\epsilon_n^*=cn^{-\beta}$.
Then we obtain the following corollary.
\begin{cor}\label{zfcor}
Let $\gamma>0$. We use Theorem~\ref{maintheo} with $G(\x,\y)=(1-\x\cdot\y)^\gamma$. 
Then \eref{approxbd} takes the form
\be\label{zfest}
\left\|f-\sum_{k=1}^N a_kG(\circ,y_k)\right\|_{\SS^Q} \le c\frac{\sqrt{\log N}}{N^{(1+\beta)/2+(2\gamma)/q}}\|\tau\|_{TV}.
\ee
\end{cor}

\subsection{Approximation on a cube by radial basis function networks}\label{laplaceexample}
We assume the set-up described in Example~\ref{cube_set_up}.
Let $G(\x,\y)=\Phi(|\x-\y|_2)$ where $\Phi$ is at least $1+Q/2$ times continuously differentiable on $[-1,1]^Q$ except at finite set $\mathcal{S}$ of  points in whose neighborhoods $\Phi$ is Lipschitz continuous. 
Apart from continuous piecewise linear functions $\Phi$, 
a typical example is $\Phi(t)=e^{-t}$.
Then $r=1$, $\mathcal{E}_\x=\{\x\}\cup \{\x-\mathcal{S}\}$, $s=0$. 
As noted in Example~\ref{gexamp}, for any $R>1$, we may choose $F(t)=ct^{-R}$. 
For any $\beta\in (0,1)$, we may choose $R> 1+\beta q/(2-2\beta)$, and $\epsilon_n^*=cn^{-\beta}$.

Theorem~\ref{maintheo} yields the following corollary.
\begin{cor}\label{cubecor}
The estimate \eref{approxbd} takes the form
\be\label{cubeest}
\left\|f-\sum_{k=1}^N a_k\Phi(|\x-\y_k|_2)\right\|_{[-1,1]^Q} \le c\frac{\sqrt{\log N}}{N^{(1+\beta)/2+1/q}}\|\tau\|_{TV}.
\ee
\end{cor}

\subsection{Manifold learning and out-of-sample extension}\label{manifold_example}

We discuss the scenario used commonly in manifold learning.
Let $\YY$ be a compact, $q$-dimensional subset of $\XX$, for example, a $q$-dimensional compact Riemannian manifold embedded in $\RR^Q$ (and hence, without loss of generality, in $[-1,1]^Q$). 
Let $\tau$ be a measure supported on $\YY$ that satisfies, in place of \eref{taucond}, the stronger condition (analogous to \eref{ballmeasure}):
\be\label{strong_taucond}
|\tau|(\BB(x,\delta))=|\tau|\left(\{y\in \XX: \rho(x,y)<\delta\}\right) \sim \delta^q, \qquad x\in\XX, \ 0<\delta\le 1.
\ee
Then we may use Theorem~\ref{maintheo} with $\YY$ in place of $\XX$, $Q=q$, $\mu^*=\tau$, and use the restrictions of the spaces $\Pi_n$ to  $\YY$. 
Then the estimate \eref{approxbd} holds with the norm taken over $\YY$ in place of $\XX$ with constants depending only on quantities related to $\YY$ without any reference to the ambient space $\XX$.

On the other hand, 
 the original estimate \eref{approxbd} is an estimate on the degree of approximation to an out-of-sample (Nystr\"om) extension of $f$ using the formula \eref{functiondef}, albeit now with constants depending upon $\XX$ as well.

In kernel based learning on manifolds, it is customary to choose a kernel $G$ defined on $\XX\times\XX$ that is infinitely smooth (e.g., the Gaussian). 
In this case,  as remarked in Remark~\ref{very_good_rem}, our estimate \eref{approxbd}  not only gives   bounds on the degree of approximation without saturation on the manifold itself, but as just remarked, also for the degree of approximation on the ambient space, without using an explicit Nystr\"om extension.

We summarize this in the following corollary.
\begin{cor}\label{manifoldcor}
In the set-up for Theorem~\ref{maintheo}, let $\YY$ be a compact subset of $\XX$,    $\tau$ be a measure on $\YY$ satisfying \eref{strong_taucond}. 
Then \eref{approxbd} takes the form
\be\label{manifoldest}
\left\|f-\sum_{k=1}^N a_kG(\circ,y_k)\right\|_\YY \le  c\left(\frac{\log N-\log \epsilon_{cN^{1/q}}^* }{N^{1+2r/q}(\epsilon_{cN^{1/q}}^*)^{(s-q)}}\right)^{1/2}\|\tau\|_{TV}.
\ee
where the constant $c$ is independent of $\XX$, and the points $y_1,\cdots,y_N\in \YY$.
If $G$ is infinitely smooth, then we have for every $S>0$,
\be\label{manifoldest1}
\left\|f-\sum_{k=1}^N a_kG(\circ,y_k)\right\|_\YY \le  c\frac{\sqrt{\log N}}{N^S}\|\tau\|_{TV},
\ee
with the same dependence of $c$ as above.
Moreover, \eref{manifoldest}, \eref{manifoldest1} hold under their respective assumptions  with $\|\cdot\|_\XX$ replacing $\|\cdot\|_\YY$ for the extension of $f$ to $\XX$ using \eref{functiondef}, except for $c$ depending on $\XX$.
\end{cor}

\bhag{Proofs.}\label{pfsect}
Our proof of Theorem~\ref{maintheo} extends to a far more general context,  some of  the ideas in \cite{bourgain1988distribution}.
The first step in this direction is to obtain a partition of $\XX$. This will be described in detail in Section~\ref{partitionsect}. 
The next step is to construct a set of random variables to which a concentration inequality can be applied. 
The basic tools for this are developed in Section~\ref{functionalsect}.
The proof of Theorem~\ref{maintheo} is then completed in Section~\ref{finalpfsect}.

\subsection{Partition of the space}\label{partitionsect}
Our main objective in this section is to prove the following theorem.

\begin{theorem}\label{partitiontheo}
Let  $\tau$ be a positive measure on $\XX$,  $\epsilon>0$, $\mathcal{A}$ be a maximal $\epsilon$-distinguishable subset of $\supp(\tau)$, and $K=\bigcup_{z\in\mathcal{A}}\BB(z,2\epsilon)$.
Then there exists a subset $\C\subseteq \mathcal{A}\subseteq \supp(\tau)$ and a partition $\{Y_y\}_{y\in\C}$ of $K$ with each of the following properties.
\begin{enumerate}
\item (\textbf{volume property}) For $y\in\C$, $Y_y\subseteq \BB(y,18\epsilon)$,  $(\kappa_1/\kappa_2)7^{-Q}\epsilon^Q\le \mu^*(Y_y)\le \kappa_2(18 \epsilon)^Q$, and\\ 
$\tau(Y_y)\ge  (\kappa_1/\kappa_2)19^{-Q}\min_{y\in\mathcal{A}}\tau(\BB(y,\epsilon))>0$.
\item (\textbf{density property}) $\eta(\C)\ge \epsilon$, $\delta(\C;K)\le 18\epsilon$.
\item (\textbf{intersection property})  Let $K_1\subseteq K$ be a compact subset. Then 
$$
\left|\{y\in \C : Y_y\cap K_1 \not=\emptyset\}\right| \le (\kappa_2^2/\kappa_1)(133)^QH_\epsilon(K_1).
$$
In particular, if $\supp(\tau)$ is a $q$-dimensional set, then $\left|\{y\in \C : Y_y\cap \supp(\tau) \not=\emptyset\}\right|\le c \epsilon^{-q}$.
\end{enumerate}
\end{theorem}

Our proof of this theorem requires some preparation, which we organize in two lemmas.
The first is the observation that among a finite collection of balls, the number of balls that can intersect each other is bounded independently of the number of balls one starts with (cf.  \cite[Lemma~7.1]{modlpmz}).

\begin{lemma}\label{noofinterlemma}
Let $\C$ be a finite subset of $\XX$, $\gamma>0$, $x\in\XX$. Then
\be\label{no_of_intersections}
\left|\{y\in \C : x\in \BB(y, \gamma\eta(\C))\}\right| \le \frac{\kappa_2}{\kappa_1}(3\gamma+1)^Q.
\ee
Thus, for any $y\in\C$, the number of balls $\{\BB(z, \gamma\eta(\C))\}_{z\in\C}$ that can have a non-empty intersection with $\BB(y, \gamma\eta(\C))$, does not exceed a fixed number, the number being independent of $\C$.
\end{lemma}

\begin{proof}\ 
In this proof, let $\eta=\eta(\C)$. 
Let $J$ be the cardinality expression on the left hand side of \eref{no_of_intersections}, and $\{y_1,\cdots,y_J\}\subseteq \C$ be such that $x\in \bigcap_{k=1}^J\BB(y_k,\gamma\eta)$.
Then $\bigcup_{k=1}^J \BB(y_k,\eta/3) \subseteq \BB(x,(\gamma+1/3)\eta)$. 
Since $\BB(y_k,\eta/3)$ are mutually disjoint, we obtain from \eref{ballmeasure} that
$$
J\kappa_13^{-Q}\eta^Q \le \sum_{k=1}^J \mu^*(\BB(y_k,\eta/3))=\mu^*\left(\bigcup_{k=1}^J \BB(y_k,\eta/3)\right)\le \mu^*(\BB(x,(\gamma+1/3)\eta))\le \kappa_2(\gamma+1/3)^Q\eta^Q.
$$
This proves \eref{no_of_intersections}.
\end{proof}

The proof of the following lemma is almost verbatim the same as that of \cite[Lemma~7.2]{modlpmz}, which in turn, is based on some ideas in the book  \cite[Appendix 1]{david2006wavelets}, but we reproduce a somewhat modified proof, both for the sake of completeness, and because the lemma was not stated in \cite{modlpmz} in the form needed here.

\begin{lemma}\label{partionlemma}
Let $K\subset \XX$, $\nu$ be a positive measure on $\XX$, $\gamma>0$. 
Let $\mathcal{A}\subset K$ be a finite set, and $\{Z_y\}_{y\in\mathcal{A}}$ be a partition of $K$ such that $Z_y\subseteq \BB(y,\gamma\eta(\mathcal{A}))$ for every $y\in\mathcal{A}$.
Then there exists a subset $\mathcal{G}\subseteq \mathcal{A}$ and a partition $\{Y_y\}_{y\in \mathcal{G}}$ of $K$ such that for each $y\in\mathcal{G}$, $Z_y\subseteq Y_y\subseteq \BB(y, 3\gamma\eta(\mathcal{A}))$, and 
$$
\nu(Y_y) \ge \frac{\kappa_1}{\kappa_2}(3\gamma+1)^{-Q}\min_{z\in \mathcal{A}}\nu(\BB(z,\gamma\eta(\mathcal{A}))).
$$
\end{lemma}

\begin{proof}\ 
In this proof, we write $\eta=\gamma\eta(\mathcal{A})$. In view of Lemma~\ref{noofinterlemma}, at most $C^{-1}=(\kappa_2/\kappa_1)(3\gamma+1)^Q$ of the balls $\BB(z_k,\eta)$ can intersect each other.  
In this proof, let $m=\min_{y\in\mathcal{A}}\nu(\BB(y,\eta))$.
If $m=0$, then the lemma is proved with $\mathcal{G}=\mathcal{A}$ with no further effort.
So, let $m>0$, and
 $\mathcal{G}=\{y\in\mathcal{A} : \nu(Z_y)\ge Cm\}$. 
Now, we define a function $\phi$ as follows. 
If $z\in {\cal G}$, we write $\phi(z)=z$. 
Otherwise, let $z\in {\cal A}\setminus {\cal G}$. 
Since $\{Z_y\}_{y\in\mathcal{A}}$ is a partition of $K$, we have
$$
m\le \nu(\BB(z,\eta)) =\sum_{y\in {\cal A}}\nu(B(z, \eta)\cap Z_y).
$$
Since each $Z_y\subseteq B(y,\eta)$, it follows that at most $C^{-1}$ of the $Z_y$'s have a nonempty intersection with  $B(z,\eta)$. So, there must exist $y\in {\cal A}$ for which
$$
\nu(B(z, \eta)\cap Z_y) \ge Cm.
$$
Clearly, each such $y\in {\cal G}$. We imagine an enumeration of ${\cal A}$, and among the $y$'s for which $\nu(B(z, \eta)\cap Z_y)$ is maximum, pick the one with the lowest index. We then define $\phi(z)$ to be this $y$. Necessarily,
$\phi(z)=y\in {\cal G}$, and $B(z, \eta)\cap Z_y\subseteq B(z, \eta)\cap B(y,\eta)$ is nonempty. So, 
\be\label{pf3eqn1}
\rho(z,\phi(z))\le 2\eta, \quad B(z, \eta)\subseteq B(\phi(z),3\eta), \quad \nu(B(z, \eta)\cap Z_{\phi(z)}) \ge Cm.
\ee
Now, we define
$$
Y_y= \bigcup\{Z_z : \phi(z)=y, z\in {\cal A}\}, \qquad y\in {\cal G}.
$$
 For each $z\in {\cal A}$, $Z_z\subseteq Y_{\phi(z)}$. Since $Z_z$ is a partition of $K$, $K=\bigcup_{y\in {\cal G}} Y_y$. If $x\in K$, $x\in Y_y\cap Y_{y'}$ for $y, y'\in {\cal G}$, then $x\in Z_z$ with $\phi(z)=y$ and $x\in Z_{z'}$ with $\phi(z')=y'$. Since $\{Z_z\}$ is a partition of $K$, it follows that $z=z'$, and hence $y=y'$. Thus, $\{Y_y\}$ is a partition of $K$, $\nu(Y_y)\ge \nu(Z_y)\ge Cm$, and 
$$
Y_y\subseteq \bigcup_{\phi(z)=y}Z_z \subseteq \bigcup_{\phi(z)=y}B(z, \eta)\subseteq B(y,3\eta). 
$$
\end{proof}
\newpage
With this preparation, we are now ready to prove Theorem~\ref{partitiontheo}.\\

\noindent\textsc{ Proof  of Theorem~\ref{partitiontheo}.}
 
Let $\mathcal{A}=\{z_1,\cdots,z_N\}$.
We set $Z_{z_1}=\BB(z_1,2\epsilon)$, and for $k=2,\cdots, N$, $Z_{z_k}=\BB(z_k,2\epsilon)\setminus \bigcup_{j=1}^{k-1}Z_{z_j}$. 
Then $\{Z_y\}_{y\in\mathcal{A}}$ is a partition of $K$ satisfying the conditions of Lemma~\ref{partionlemma} with $\gamma=2$.

We apply Lemma~\ref{partionlemma} first with $\mu^*$ in place of $\nu$. 
This yields a subset $\mathcal{G}_1\subset \mathcal{A}$ and a partition $\{\widetilde{Y}_y\}_{y\in\mathcal{G}_1}$ of $K$ such that for each $y\in \mathcal{G}_1$,  $Z_y\subseteq \widetilde{Y}_y\subset \BB(y, 6\epsilon)$, and $\mu^*(\widetilde{Y}_y)\ge (\kappa_1/\kappa_2)7^{-Q}\epsilon^Q$.
Clearly, $\eta(\mathcal{G}_1)\ge \eta(\mathcal{A})\ge \epsilon$.
We apply Lemma~\ref{partionlemma} again with $\tau$ in place of $\nu$, $\mathcal{G}_1$ in place of $\mathcal{A}$, $\gamma=6$. 
This yields $\C\subseteq \mathcal{G}_1$ and a partition $\{Y_y\}_{y\in\C}$ of $K$ such that for each $y\in \C$,  $\widetilde{Y}_y\subseteq Y_y\subset \BB(y, 18\epsilon)$ and $\tau(Y_y)\ge (\kappa_1/\kappa_2)19^{-Q}\min_{y\in\mathcal{G}_1}\tau(\BB(y,\epsilon))$.
Since $\mathcal{G}_1\subseteq \mathcal{A}\subseteq \supp(\tau)$, the minimum expression is positive.  Moreover, 
\be\label{pf1eqn1}
\mu^*(Y_y) \ge\mu^*(\widetilde{Y}_y) \ge (\kappa_1/\kappa_2)7^{-Q}\epsilon^Q.
\ee
This proves the volume and density properties.

In this proof, let $\mathcal{D}=\{y\in \C : Y_y\cap K_1 \not=\emptyset\}$, and 
 $\{x_1,\cdots, x_L\}$ be a maximal $\epsilon$-distinguishable subset of $K_1$. Clearly, each $x_k$ belongs to some $Y_y$. 
Next, let $\widetilde{K} =\mathbb{B}(K_1, 18\epsilon)$. Since $K_1\subseteq \bigcup_{k=1}^L \BB(x_k,\epsilon)$, it is clear that
$\widetilde{K}\subseteq \bigcup_{k=1}^L \BB(x_k,19\epsilon)$. 
In view of \eref{ballmeasure},
$\mu^*(\widetilde{K}) \le \kappa_2L(19\epsilon)^Q$. 
On the other hand, since  $Y_y\subseteq \BB(y,18\epsilon)$ for each $y\in \mathcal{D}$, $\bigcup_{y\in \mathcal{D}}Y_y \subseteq \widetilde{K}$. Using  the volume property proved already, we deduce that 
$$
|\mathcal{D}|(\kappa_1/\kappa_2)7^{-Q}\epsilon^Q\le \sum_{y\in\mathcal{D}}\mu^*(Y_y)= \mu^*\left(\bigcup_{y\in \mathcal{D}}Y_y\right)\le \mu^*(\widetilde{K})\le \kappa_2L(19\epsilon)^Q.
$$
Thus, $|\mathcal{D}|\le (\kappa_2^2/\kappa_1)(133)^QL =(\kappa_2^2/\kappa_1)(133)^QH_\epsilon(K_1)$. 
\qed

\subsection{Construction of probability measures}\label{functionalsect}

The main purpose of this section is to obtain the construction of a probability measure in an abstract setting.
In the proof of Theorem~\ref{maintheo}, we will use this construction with each of the elements of the partition which we developed in Section~\ref{partitionsect}, and use Hoeffding's inequality.
Thus, the main objective in this section is to prove the following theorem.

\begin{theorem}\label{bourgaintheo}
Let $\YY$ be a compact topological space, $\{\psi_j\}_{j=0}^{M-1}$ be continuous real valued functions on $\YY$, and $\nu$ be a probability measure on $\YY$.
Let $\mathbb{P}_M(\YY)$ denote the set of all probability measures $\omega$ supported on at most $M+2$ points of $\YY$ with the property that
\be\label{abs_quadrature}
\int_\YY \psi_j(y)d\omega(y)=\int_{\YY}\psi_j(y)d\nu(y), \qquad j=0,\cdots, M-1.
\ee
Then $\nu$ is in the weak-star closed convex hull of $\mathbb{P}_M(\YY)$, and hence, there exists a measure $\omega^*_\YY$ on $\mathbb{P}_M(\YY)$ with the property that for any $f\in C(\YY)$,
\be\label{barycenter}
\int_\YY f(y)d\nu(y)=\int_{\mathbb{P}_M(\YY)} \left(\int_\YY f(y)d\omega(y)\right)d\omega^*_\YY(\omega).
\ee
\end{theorem}

The starting point of the proof of this theorem is to  recall the following theorem, called Tchakaloff's theorem \cite[Exercise~2.5.8, p.~100]{rivlin1974chebyshev}.

\begin{theorem}\label{tchakalofftheo}
Let $\YY$ be a compact topological space, $\{\psi_j\}_{j=0}^{M-1}$ be continuous real valued functions on $\YY$, and $\nu$ be a probability measure on $\YY$. Then there exist $M+1$ points $z_1,\cdots,z_{M+1}$, and non--negative numbers $w_1,\cdots,w_{M+1}$ such that
\begin{equation}\label{tchakaloffquad}
\sum_{k=1}^{M+1} w_k=1, \qquad \sum_{k=1}^{M+1}w_k\psi_j(z_k)=\int_\YY \psi_j(z)d\nu(z), \qquad j=0,\cdots,M-1.
\end{equation}
\end{theorem}

In the following proof of Theorem~\ref{bourgaintheo}, 
 $C(\YY)^*$ denotes the dual space of $C(\YY)$, equipped with the weak-star topology. \\

\noindent\textsc{Proof Theorem~\ref{bourgaintheo}.}

If $\nu$ is not in the the weak-star closed convex hull of $\mathbb{P}_M(\YY)$, then there exists a $g\in C(\YY)$ and $\delta>0$ such that the interval of radius $\delta$ centered at $\int_\YY g(y)d\nu(y)$ does not contain any element of the set $\{\int_\YY g(y)d\omega(y) : \omega\in \mathbb{P}_M(\YY)\}$.
However,  Thereom~\ref{tchakalofftheo} applied with the system $\{\psi_0,\cdots,\psi_{M-1}, g\}$ shows that there exists $\omega\in \mathbb{P}_M(\YY)$ such that $\int_\YY g(y)d\nu(y)=\int_\YY g(y)d\omega(y)$. 
This contradiction shows that $\nu$ is  in the the weak-star closed convex hull of $\mathbb{P}_M(\YY)$.
Necessarily, this closed convex hull is weak-star compact.
Also, it is clear that $\mathbb{P}_M(\YY)$ is also weak-star compact.
Since $C(\YY)\subset (C(\YY)^*)^*$  separates points in $C(\YY)^*$, then we may apply \cite[Theorem~3.28]{rudinfunctional} to complete the proof.
\qed

We end this section by recalling the Hoeffding's inequality \cite[Appendix~B, Corollary~3]{pollard2012bk}.
\begin{lemma}\label{hoeffdinglemma}
Let $X_1,\cdots, X_n$ be independent random variables with zero means and bounded ranges: $a_j\le X_j\le b_j$, $j=1,\cdots,n$. 
Then 
\be\label{hoeffdingineq}
\mathsf{Prob}\left(\left|\sum_{j=1}^N X_j\right| \ge t\right)\le 2\exp\left(-\frac{2t^2}{\sum_{j=1}^N (b_j-a_j)^2}\right), \qquad t>0.
\ee
\end{lemma}

\subsection{Proof of Theorem~\ref{maintheo}}\label{finalpfsect}

In view of the Jordan decomposition of $\tau$, there is no loss of generality in assuming that $\tau$ is a probability measure.
Let $n\ge 1$ be an integer, and $\mathcal{A}$ be a $1/(2n)$-distinguishable subset of $\supp(\tau)$. 
We observe that the set $S_n=\BB(\mathcal{A},1/n)$ is compact, and $\supp(\tau)\subset S_n$.  
Therefore, we may find $\C\subset \supp(\tau)$ with $|\C|\le cn^q$ and a partition $\{Y_y\}_{y\in\C}$ of $S_n$ satisfying all the conclusions of Theorem~\ref{partitiontheo} with $\epsilon=1/(2n)$.
Let $\mathcal{Y}$ be the set, each of whose element is a finite intersection of the sets from $\{\overline{Y_y}\}_{y\in\C}$ having positive $\tau$-measure.
Each element of $\mathcal{Y}$ is a compact subset of $S_n$, and in view of Lemma~\ref{noofinterlemma}, $|\mathcal{Y}|\le cn^q$.

Now, we apply Theorem~\ref{bourgaintheo} to each element $A\in \mathcal{Y}$ with $D_R$ in place of $M$, a basis $\{\psi_j\}$ of $\Pi_R$,  and $\tau_A=\frac{1}{\tau(A)}\tau$ in place of $\nu$. 
For each $A$, this gives a measure $\omega^*_A$ on $\mathbb{P}_{D_R+2}(A)$ such that
\be\label{pf2eqn1}
\int_A fd\tau_A=\int_{\mathbb{P}_{D_R+2}(A)} \left(\int_A fd\omega\right)d\omega^*_A, \qquad f\in C(A);
\ee
in particular,
\be\label{pf2eqn2}
\int_A Pd\tau =\int_{\mathbb{P}_{D_R+2}(A)} \left(\tau(A)\int_A Pd\omega\right)d\omega^*_A, \qquad P\in\Pi_R.
\ee

Next, we consider a family of independent random variables. Let $x\in\XX$, and $A\in \mathcal{Y}$ be the intersection of exactly $k$ of the sets $\{\overline{Y_y}\}_{y\in\C}$. 
We define a random variable on $\mathbb{P}_{D_R+2}(A)$ having $\omega^*_A$ as the probability law by
$$
\Omega_A(\omega)=(-1)^{k-1}\left(\tau(A)\int_A G(x,y)d\omega(y)-\int_A G(x,y)d\tau(y)\right), \qquad \omega\in \mathbb{P}_{D_R+2}(A).
$$
For any realization of these random variables, $\omega_A\in \mathbb{P}_{D_R+2}(A)$, we write
\be\label{pf2eqn8}
\mathbb{G}(\{\omega_A\};x)=\mathbb{G}(x)=\sum_A (-1)^{k-1}\tau(A)\int_A G(x,y)d\omega_A(y).
\ee
Then
\bea\label{pf2eqn6}
\sum_A \Omega_A(\omega_A)&=&\sum_A (-1)^{k-1}\tau(A)\int_A G(x,y)d\omega_A(y)- \sum_A (-1)^{k-1}\int_A G(x,y)d\tau(y)\nonumber\\
&=&\mathbb{G}(x)-\int_{\bigcup_{y\in\C} \overline{Y_y}}G(x,y)d\tau(y)\nonumber\\
&=& \mathbb{G}(x)-\int_\XX G(x,y)d\tau(y)=\mathbb{G}(x)-f(x).
\eea
We will estimate the probability that $\left|\sum_A \Omega_A(\omega_A) \right|$ is $\ge t$ for a $t>0$ to be chosen later.

In order to apply Hoeffding's inequality \eref{hoeffdingineq}, we observe first using \eref{pf2eqn1} that the expected value of each $\Omega_A$ with respect to $\omega_A^*$ is $0$. We need to estimate the sum of squares expression in \eref{hoeffdingineq}.

In view of the definition of $\mathbb{P}_{D_R+2}(A)$, we see that for every $\omega\in \mathbb{P}_{D_R+2}(A)$ and $P\in\Pi_R$,
$$
\tau(A)\int_A G(x,y)d\omega(y)-\int_A G(x,y)d\tau(y)= \tau(A)\int_A \left(G(x,y)-P(y)\right)d\omega(y)-\int_A \left(G(x,y)-P(y)\right)d\tau(y).
$$
Now, for each $y\in A$, $A\subseteq \BB(y, c/n)$. 
Hence, we deduce that for every $\omega\in \mathbb{P}_{D_R+2}(A)$,
\be\label{pf2eqn3}
|\Omega_A(\omega)|\le c\tau(A)\sup_{y\in A} E_R(\BB(y,c/n);G(x,\circ)).
\ee
Therefore, using \eref{taucond}, \eref{globalsmooth}, \eref{sobolevnormbd}, we conclude that
\be\label{pf2eqn4}
|\Omega_A(\omega)|\le \beta_A= cn^{-q}\begin{cases}
n^{-r}, &\mbox{ if $A\cap \BB(\mathcal{E}_x,\epsilon_n^*)\not=\emptyset$,}\\
F(\epsilon_n^*)n^{-R}, &\mbox{otherwise}.
\end{cases}
\ee
It is clear that the number of the sets $A$ in $\mathcal{Y}$ for which $A\cap \BB(\mathcal{E}_x,\epsilon_n^*)=\emptyset$ is at most $cn^q$. Since  the family $\{\mathcal{E}_x\cap\supp(\tau)\}_{x\in\XX}$ is $s$-dimensional, there are at most $c(\epsilon_n^*)^{-s}$ balls of radius $2\epsilon_n^*$ that cover $\BB(\mathcal{E}_x,\epsilon_n^*)\cap\supp(\tau)$, with $c$ independent of $x$.
Since each $A$ for which $A\cap \BB(\mathcal{E}_x,\epsilon_n^*)\not=\emptyset$ is contained in a ball of radius $c/n$, we deduce using the construction of the partition that the number of such $A$'s is at most $c(n\epsilon_n^*)^q(\epsilon_n^*)^{-s}$.
Consequently, 
\be\label{pf2eqn5}
\sum_A \beta_A^2 \le c \left\{\sum_{A : A\cap \mathcal{E}_x\not=\emptyset} n^{-2q-2r} + F(\epsilon_n^*)^2\sum_{A : A\cap \mathcal{E}_x=\emptyset} n^{-2q-2R}\right\}\le cn^{-q-2r}\left\{(\epsilon_n^*)^{q-s} + F(\epsilon_n^*)^2n^{2r-2R}\right\}.
\ee
Our choice of $\epsilon_n^*$ shows that 
$$
(\epsilon_n^*)^{q-s}\ge F(\epsilon_n^*)^2n^{2r-2R},
$$
so that
$$
\sum_A \beta_A^2 \le cn^{-q-2r}(\epsilon_n^*)^{q-s}.
$$
Consequently, Hoeffding's inequality (Lemma~\ref{hoeffdinglemma}) implies that for each $x\in\XX$,
\be\label{pf2eqn7}
\mathsf{Prob}\left(|\mathbb{G}(x)-f(x)|\ge t\right)\le 2\exp(-ct^2n^{2r+q}(\epsilon_n^*)^{s-q}).
\ee

Next, writing 
$$
\varepsilon= \left(n^{-2r-q}(\epsilon_n^*)^{q-s}\right)^{1/\alpha},
$$
 we choose a maximal $\varepsilon$-distinguishable set $\C'\subseteq \XX$, and apply \eref{pf2eqn7} with each element of $\C'$. 
Since $|\C'|\sim \varepsilon^{-Q}$, we obtain that
\be\label{pf2eqn9}
\mathsf{Prob}\left(\max_{x'\in\C'}|\mathbb{G}(x')-f(x')|\ge t\right)\le 2|\C'|\exp(-ct^2n^{2r+q}(\epsilon_n^*)^{s-q}) \le c_1\varepsilon^{-Q}\exp(-ct^2n^{2r+q}(\epsilon_n^*)^{s-q}).
\ee 
Since $\XX= \bigcup_{x'\in \C'}\BB(x',\varepsilon)$, it follows that for every $x\in\XX$, there exists $x'\in \C$ with $\rho(x,x')\le \varepsilon$. The condition \eref{lipschitzcond} then leads to the fact that if $\nu$ is any probability measure on $\XX$, then
\be\label{pf2eqn12}
\left|\int_\XX G(x,y)d\nu(y)-\int_\XX G(x',y)d\nu(y)\right|\le c\varepsilon^\alpha=cn^{-2r-q}(\epsilon_n^*)^{q-s}.
\ee
In particular, for every $A$ and $\omega\in \mathbb{P}_{D_R+2}(A)$,
\be\label{pf2eqn10}
|f(x)-f(x')| \le cn^{-2r-q}(\epsilon_n^*)^{q-s}, \qquad \left|\int_A G(x,y)d\omega(y)-\int_A G(x',y)d\omega(y)\right| \le cn^{-2r-q}(\epsilon_n^*)^{q-s}.
\ee
Next, we observe that the number of $A$'s that can intersect each other is $\le c$, and hence 
$$
\sum_A \tau(A) \le c_1\tau(\bigcup A) \le c_2\tau(\XX)=c_3.
$$
Hence, the second set of inequalities in \eref{pf2eqn10} and the definition \eref{pf2eqn8} lead to 
$$
|\mathbb{G}(x)-\mathbb{G}(x')| \le cn^{-2r-q}(\epsilon_n^*)^{q-s}.
$$
Together with \eref{pf2eqn10}, this implies that
$$
\left|\|f-\mathbb{G}\|_\XX -\max_{x'\in \C'}|f(x')-\mathbb{G}(x')|\right|\le cn^{-2r-q}(\epsilon_n^*)^{q-s}.
$$
Therefore, \eref{pf2eqn9} leads to
\be\label{pf2eqn11}
\mathsf{Prob}\left(\|f-\mathbb{G}\|_\XX>t + cn^{-2r-q}(\epsilon_n^*)^{q-s}\right)\le c_1\varepsilon^{-Q}\exp(-ct^2n^{2r+q}(\epsilon_n^*)^{s-q}).
\ee
Choosing 
$$
t=\frac{c_2}{n^{r+q/2}(\epsilon_n^*)^{(s-q)/2}}\sqrt{\log (\varepsilon^{-1})}
$$
for a judicious choice of $c_2$, the right hand side of \eref{pf2eqn11} is $<1/2$. 
This shows that there exists a choice of $\omega_A$'s such that 
$$
\|f-\mathbb{G}(\{\omega_A\};\circ)\|_\XX \le \frac{c_2}{n^{r+q/2}(\epsilon_n^*)^{(s-q)/2}}\sqrt{\log (\varepsilon^{-1})}=c\left(\frac{\log N-\log \epsilon_{cN^{1/q}}^* }{N^{1+2r/q}(\epsilon_{cN^{1/q}}^*)^{(s-q)}}\right)^{1/2}.
$$
Finally, we note that 
$$\mathbb{G}(\{\omega_A\};x)=\sum_A (-1)^{k-1}\tau(A)\sum_{j=1}^{D_R+2}b_{j,A}G(x-x_{j,A}),
$$
where $b_{j,A}\ge 0$, $\sum_{j=1}^{D_R+2}b_{j,A}=1$ and $x_{j,A}\in S_n$. 
Since the number of $A$'s with a common intersection does not exceed $c$, and $\tau(A)\le c/n^q$, it follows that $\mathbb{G}(\{\omega_A\};\circ)$  is a $G$-network with at most $c(D_R+2)n^q$ terms, which can be expressed in the form given in \eref{approxbd} with the coefficients satisfying $|a_k|\le c/n^q\sim 1/N$.
\qed

\bhag{Conclusions}\label{concludesect}
We have proved an abstract theorem that addresses in a  unified manner the following two questions in the theory of machine learning: (1) dimension independent bounds on the degree of approximation by linear combinations of a kernel $G$, and (2) bounds on the degree of approximation on the out-of-sample Nystr\"om extension for a class of functions by networks trained on a compact subset of the ambient space. 
We are also interested in another problem in the area of information based complexity: tractability of integration in non-tensor product domains.

We have given a very general theorem, Theorem~\ref{maintheo}, to answer all of these questions in one stroke.
Our theorem combines the best aspects of both the probabilistic approach typically used in order to get dimension independent bounds given a kernel representation, as well as the classical approximation theory approach to ensure that the smoother the target function, the better is the rate of approximation, without saturation.

Necessarily, the theorem is rather abstract, but we have illustrated a few applications. In particular, we have developed  dimension independent bounds on ReLU-type networks on the sphere (Corollary~\ref{relucor}), and hence, on the Euclidean space as explained in \cite{bach2014, dingxuanpap}. 
In Corollary~\ref{cubecor}, we have given similar bounds for approximation by certain radial basis function networks. 
The error bounds in the context of manifold learning as well the the bounds on an out-of-sample extension are given in Corollary~\ref{manifoldcor}.

We have argued in \cite{dingxuanpap} that the superiority of deep networks over shallow networks stems from the fact that deep networks can utilize any compositional structure in the target function, thereby mitigating the curse of dimensionality by the ``blessing of compositionality''. 
Our results above indicate that the degree of approximation alone, without any requirement for robust parameter selection, is not adequate to explain this superiority. 
For example, for a deep ReLU network with a binary tree structure receiving 1024 inputs, the degree of approximation by itself would give an accuracy of (up to a logarithmic term) $\O(N^{-1.25})$, while the same for a shallow network is $\O(N^{-0.5015})$. 
When a robust parameter selection is required then the two estimates are $\O(N^{-1})$ and $\O(N^{-0.002})$ respectively.

We point out a philosophical comment. In general, the usual machine learning paradigm works with a split of the generalization error into bias and variance term, with a further split of the bias into approximation error (degree of approximation) and sampling error (empirical risk minimization). 
Usually, one pays attention only to the degree of approximation, ignoring any details of how it is achieved. 
This estimate is then used merely as a guideline for setting up the empirical risk minimization problem, whose solution has nothing to do with the minimizer in the degree of approximation estimation.
In turn, this requires a trade-off. 
The degree of approximation gets better with the increase in the number of parameters, but the complexity of the minimization problem gets higher.
Therefore, a careful balance is required.
Our result calls into question  this paradigm, pointing out the need to impose some further conditions on the estimate on the degree of approximation, hopefully, prompting the development of a new paradigm where the split between approximation and sampling errors is no longer necessary.
In the case when the marginal distribution of the independent variable is known to be supported on a compact, smooth, Riemannian manifold, we have developed a full theory of machine learning without using this split (e.g., \cite{mauropap, eignet, heatkernframe, compbio}).

Coming to the question of tractability of integration, there is a vast amount of literature investigating conditions under which quadrature formulas based on $N$ points can be constructed for integration with respect to a tensor product weight on a high dimensional cube, so as to achieve an error bound of the form (up to logarithmic terms) $c_1N^{-c}$, where $c$ is independent of the dimension of the cube, and $c_1$ is dependent at most polynomially on the dimension.
Although the functions to be integrated do have the form \eref{functiondef}, the bounds in the literature typically do not improve with the smoothness of the functions.
Corollary~\ref{tractablecor} shows the existence of quadrature formulas for far more general domains and measures, without any tensor product structure, which give the error bounds for integration which are of the form $\O(N^{-c})$, although the constants may depend upon the dimension.
Moreover, the bounds are better for smoother functions.


\end{document}